\documentclass[letterpaper]{article} 
\usepackage{aaai2026}
\usepackage{times}  
\usepackage{helvet}  
\usepackage{courier}  
\usepackage[hyphens]{url}  
\usepackage{graphicx} 
\usepackage{natbib}  
\usepackage{caption} 
\usepackage{algorithm}
\usepackage{algorithmic}

\usepackage{newfloat}
\usepackage{listings}
\DeclareCaptionStyle{ruled}{labelfont=normalfont,labelsep=colon,strut=off} 
\floatstyle{ruled}
\newfloat{listing}{tb}{lst}{}
\floatname{listing}{Listing}
\usepackage[utf8]{inputenc}
\newtheorem{definition}{Definition}
\usepackage{amssymb,amsmath,graphics,url}
\usepackage{booktabs,colortbl}
\newcommand{\A}{{\sf A}}
\newcommand{\K}{{\sf K}}
\newcommand{\R}{{\sf R}}
\newcommand{\D}{{\sf D}}

\renewcommand{\phi}{\varphi}
\renewcommand{\epsilon}{\varepsilon}

\newtheorem{lemma}{Lemma}
\newtheorem{claim}{Claim}
\newtheorem{corollary}{Corollary}
\newtheorem{theorem}{Theorem}

\newenvironment{proof}{\begin{trivlist}\item\noindent{\em Proof.}}{\hfill {$\boxtimes$}\end{trivlist}}

\newenvironment{proof-of-claim}{\begin{trivlist}\item\noindent{\em Proof of Claim.}}{\hfill {\tiny $\boxtimes$}\end{trivlist}}

\usepackage{upgreek}
\renewcommand{\alpha}{\upalpha}
\renewcommand{\beta}{\upbeta}
\usepackage{fontawesome}
\usepackage{stmaryrd}
\renewcommand{\[}{\llbracket}
\renewcommand{\]}{\rrbracket}

\title{De Re and De Dicto Awareness in Multiagent Systems}

\title{An Epistemic Perspective on Agent Awareness}

\nocopyright

\author {
    Pavel Naumov\textsuperscript{\rm 1},
    Alexandra Pavlova\textsuperscript{\rm 2}
}
\affiliations {
    \textsuperscript{\rm 1}University of Southampton, United Kingdom\\
    \textsuperscript{\rm 2}TU Wien, Austria\\
    p.naumov@soton.ac.uk, alexandra@logic.at
}

\begin{document}

\maketitle

\begin{abstract}
The paper proposes to treat agent awareness as a form of knowledge, breaking the tradition in the existing literature on awareness. 
It distinguishes the de re and de dicto forms of such knowledge. The work introduces two modalities capturing these forms and formally specifies their meaning using a version of 2D-semantics. The main technical result is a sound and complete logical system describing the interplay between the two proposed modalities and the standard ``knowledge of the fact'' modality.     
\end{abstract}

\section{Introduction}

Artificial agents are increasingly making important decisions that affect our lives.
The choice of the right decision often depends on the {\em awareness} about other agents' presence at the scene. A war robot must minimise casualties if it is aware of civilians present at the theatre of operations. A self-driving car must stop at the yield (give way) sign if it is aware of an approaching vehicle on the other road.
An agent, who is a medical doctor, must offer help if she is aware of someone being sick. An autonomous driving system must stop the vehicle if it is aware of an approaching emergency vehicle. A machine whose values align with humans' must apologise if it is aware of someone being offended. A well-mannered robot should not take the last piece of a cake if it is aware of somebody else wanting this piece.

Awareness is a vague term that could be interpreted in different ways. In the literature, the authors consider awareness of an object~\cite*{bc22jmid,bc21jmid,bcs11synthese} and conceptual awareness~\cite*{fh87ai,vfvw13tark,vv10synthese,gv15synthese,s15hel}. In this work, we focus on a specific form of the former: ``agent awareness'' or awareness of one agent about the existence of another agent with a certain property. Although the existing literature on awareness treats it as a distinct concept, the Cambridge Dictionary suggests an epistemic interpretation of awareness by defining it as ``knowledge that something exists''. In this paper, we give a formal account of this epistemic approach to awareness. While doing this, we observe that there are two distinct ways in which one agent can be aware of another.

Indeed, let us consider the following two sentences:
\begin{quote}
\em Vehicle's autopilot started to slow down after it became {\bf\em  aware} of a AAAI attendee crossing the road.    
\end{quote}
\begin{quote}
\em Vehicle's autopilot started to slow down after it became {\bf\em aware} of being followed by a police car.   
\end{quote}
Note that, in these sentences, the autopilot is aware in two distinct senses. In the first sentence, it knows that there exists a human on the road ahead of the vehicle. The human happens to be a AAAI attendee, but the autopilot does not necessarily know this. The awareness is about {\em the physical object} (human on the road) and not her property of being a AAAI attendee. Using first-order epistemic logic, we can write this as
$$
\exists x(\text{AAAI-Attendee}(x) \wedge \K_{\text{autopilot}} \text{CrossingRoad}(x)).
$$
In the second sentence, the autopilot hears the siren and knows that one of the vehicles driving behind is a police car, but it might not even know which of the vehicles behind belongs to the police. This is awareness of someone with {\em designator} ``police car'' being behind: 
$$
\K_{\text{autopilot}}\, \exists x(\text{PoliceCar}(x) \wedge\text{DrivingBehind}(x)).
$$

In the philosophy of language, a distinction between a reference to an object and to a designator of this object is usually referred to as a {\em de re/de dicto} distinction. 
Following this tradition, we say that the autopilot is aware of the AAAI attendee {\em de re} and of the police car {\em de dicto}.
De re/de dicto distinction (without the context of awareness) has been the subject of studies in the philosophy of language~\cite{q56jp,l79pr,c76ps,a97lp,ks19ohr} and law~\cite{a14hlr,y11pfll}. \citet{rt11pq} introduced a related notion ``de objecto''.
The existing approaches to formally capturing the distinction mostly rely on quantifiers.

As another example, consider the sentences:
\begin{quote}\em
A robot series {\tt i17} must introduce itself if a security guard becomes {\bf\em aware} of its presence
\end{quote}
\begin{multline*}
\forall x(\text{\tt i17}(x) \wedge \K_{\text{guard}} (\text{Present}(x))\!\to\!\text{MustSelfIntro}(x)),    
\end{multline*}
and
\begin{quote}\em
A security guard must report to a supervisor any {\bf\em sighting} of a series {\tt i17} robot
\end{quote}
\begin{multline*}
\forall x( \K_{\text{guard}} (\text{Present}(x) \wedge\text{\tt i17}(x))\!\to\!\text{MustReport}(\text{guard})),    
\end{multline*}
the first refers to {\em de re} awareness of the security guard about the robot as a physical object (not necessarily of {\tt i17} series). Otherwise, what would be the point of an introduction if the guard already knows that the robot is of series {\tt i17}? The word ``sighting'' in the second sentence refers to de dicto awareness. The guard must report to a supervisor if the guard knows that the robot is of that specific series.

Finally, the sentence
\begin{quote}\em
A robot of series {\tt i17} must self-distract if it becomes {\bf\em aware} of someone {\bf\em aware} of its presence
\end{quote}
\begin{multline*}
\forall x (\text{\tt i17}(x) \wedge \exists y(y\neq x \wedge \K_{x}\K_{y}(\text{Present}(x))) \\\to\text{MustSelfDistract}(x))
\end{multline*}
mentions awareness twice. The first of them is the awareness of $X$ (robot series {\tt i17}) that some other agent $Y$ nearby has the property of ``being aware of $X$'s presence''. This is a de dicto awareness. The second of them is $Y$'s de re awareness of $X$ as a physical object, perhaps without knowing that $X$ is a robot series {\tt i17}.   

Even without focusing specifically on awareness, very few quantifier-free logical systems for capturing the de re/de dicto distinction have been proposed. \citet{ent23jlc} considered modalities that capture de re/de dicto versions of ``know who''. Epistemic Logic with Assignments~\cite{ws18aiml,ctw21tark,wws22apal} proposes a very general language that can also be used to capture de re and de dicto knowledge of one agent about knowledge of the other. \citet{jn25cyber} proposed a modal logical system for reasoning about de re and de dicto knowledge of a property of an agent inferred from a dataset. However, none of these logical systems can express either de re or de dicto awareness.

In this paper, we propose to capture de re and de dicto forms of awareness by two modalities, whose meaning is defined using a 2D-semantics~\cite{s21sep}. Our main technical result is a sound and complete logical system capturing the interplay between the traditional ``knowledge of the fact'' modality and these two new modalities. The completeness proof uses a non-trivial modification of a recently introduced ``matrix'' technique.

The proposed logical system complements the existing body of literature on other specialised forms of knowledge:
know-how~\cite*{nt17aamas,fhlw17ijcai,aa16jlc},
know-who~\cite{ent23jlc},
know-whether~\cite{fwv15rsl},
know-why~\cite{xws19synthese}, and
know value~\cite{wf13ijcai,b16aiml}.

\section{Running Example}

Imagine Ann, who decided to take a break from AAAI-26 meetings at 
Singapore's {\em Asian Civilisations Museum}.
While in the building, she booked a WeRide self-driving car to take her back to the conference. As Ann leaves the building, she notices a car parked in front of the entrance to the building. Unknown to Ann, the car in front of her is an unmarked police vehicle. Ann is aware of the car in front of her. The car is a police vehicle. Thus, {\em Ann is de re aware of the police vehicle next to the museum}.

A few seconds after exiting the building, Ann gets a text message that her WeRide has arrived. Ann cannot see the WeRide vehicle because it is parked around the corner, but Ann knows that the WeRide is somewhere near her. Hence, {\em Ann is de dicto aware of the WeRide vehicle next to the museum}.
In this paper, we capture these two forms of awareness by two modalities.
We believe that the definitions of these modalities are the most elegant in the {\em egocentric} logic setting. The semantics of the traditional modal logical systems is defined in terms of a binary relation $w\Vdash\phi$. In such a setting, formula $\phi$ captures a property of possible world $w$. \citet{p68nous} proposed to consider logics that capture properties of {\em agents} rather than possible {\em worlds}. He called such logics ``egocentric''. The semantics of egocentric logical systems can be defined in terms of a binary relation $a\Vdash \phi$ between an agent $a$ and a formula $\phi$. Multiple versions of such systems, not dealing with awareness, have been proposed in the literature~\cite{gh91kr,gh93jlc,g95ai,slg11lia,slg13tark,ch15jal,chp16jlli,jn22ijcai-preferences,jn24synthese-preferences}.  

In order to capture knowledge and awareness, one can extend egocentric semantics by considering a ternary satisfaction relation $w,a\Vdash \phi$ between possible world $w$, agent $a$, and formula $\phi$. In such a setting, formula $\phi$ captures property $\phi$ of agent $a$ in world $w$. 

\begin{figure}[ht]
\begin{center}
\vspace{2mm}
\scalebox{0.55}{\includegraphics{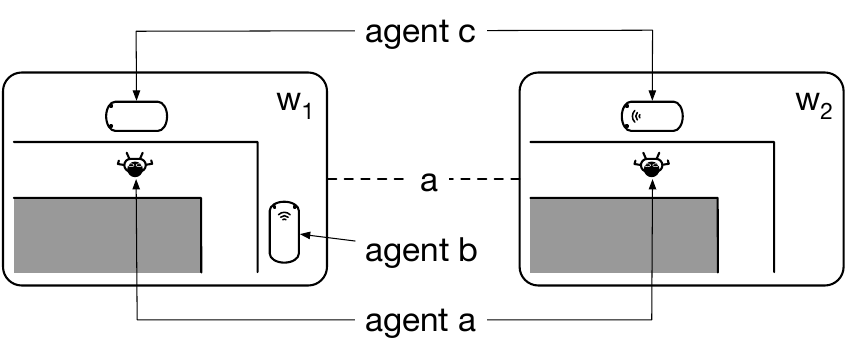}}
\caption{Symbol {\scriptsize \faWifi\ } designates a WeRide vehicle. }\label{intro figure}
\end{center}
\vspace{0mm}
\end{figure}
For example, consider an {\em epistemic model}, depicted in Figure~\ref{intro figure}, capturing the setting of our running example. This model has two worlds, $w_1$ and $w_2$ with $w_1$ being the actual world in the example. There are three agents {\em present} in world $w_1$: Ann (agent $a$), WeRide vehicle (agent $b$), unmarked police vehicle (agent $c$). In Figure~\ref{intro figure}, two-directional arrows represent transworld identity between instances of agents in different worlds. Although the nature and the very existence of transworld identity is a widely discussed subject in the philosophy of language~\cite{mj22sep}, in this work, we assume such identity to be given.

Because agent $c$ is a police vehicle in world $w_1$ we can write: 
\begin{equation*}
w_1,c\Vdash \text{``is police vehicle''}.
\end{equation*}
Furthermore, because agent $c$ is located near the museum, 
\begin{equation*}
w_1,c\Vdash \text{``is police vehicle''} \wedge \text{``is near the museum''}.
\end{equation*}

Ann cannot distinguish world $w_1$ from world $w_2$. However, in world $w_2$ the same agent $c$ is a WeRide vehicle located near the museum:
\begin{equation*}
w_2,c\Vdash \text{``is a WeRide vehicle''} \wedge \text{``is near the museum''}.
\end{equation*}
Because the same agent $c$ is present in both worlds that Ann cannot distinguish, Ann is aware of agent $c$. Since in the actual world agent $c$ is an unmarked police vehicle, in world $w_1$ Ann is {\em de re}  (as of a physical object) aware of the police vehicle near the museum. We write this as
\begin{equation*}
w_1,a\Vdash \R(\text{``is police vehicle''} \wedge \text{``is near the museum''}).
\end{equation*}
Recall that agent $b$ is also present in the world $w_1$ and, in this world, it is a WeRide vehicle located near the museum:
\begin{equation*}
w_1,b\Vdash \text{``is a WeRide''} \wedge \text{``is near the museum''}.
\end{equation*}
Because Ann cannot see agent $b$, she is not aware of it de re (as a physical object):
\begin{equation*}
w_1,a\Vdash\neg \R(\text{``is a WeRide''} \wedge \text{``is near the museum''}).
\end{equation*}
At the same time, because Ann got the message from WeRide that her vehicle had arrived at the museum, there must be a WeRide near the museum in each world that Ann cannot distinguish from the current world. In our example, there is a WeRide near the museum in world $w_1$ (agent $b$) as well as in world $w_2$ (agent $c$). As a result, Ann is aware of a WeRide, as a concept (de dicto), being present near the museum:
\begin{equation*}
w_1,a\Vdash \D(\text{``is a WeRide''} \wedge \text{``is near the museum''}).
\end{equation*}

The rest of the paper is structured as follows. First, we introduce epistemic models and proceed by defining the syntax and semantics of our logical system. Then we propose the axiomatisation. Having informally discussed the axioms, we state the soundness theorem. In the next section, we introduce the notions of general awareness and a $\lambda$-{\bf\em assured} set. Finally, we prove the completeness of the logical system using the ``matrix'' technique.

\section{Epistemic Models}

In this section, we define the class of models that we use later to define the formal semantics of our logical system. Throughout the paper, we assume a fixed nonempty set of propositional variables.

\begin{definition}\label{epistemic model definition}
\hspace{-2mm} A tuple $\!(W,\mathcal{A},P,\sim,\pi)\!$ is an epistemic model if
\begin{enumerate}
    \item $W$ is a (possibly empty) set of all ``worlds'',
    \item $\mathcal{A}$ is a (possibly empty) set of ``agents'',
    \item $P\subseteq \mathcal{A} \times W$ is a ``presence'' relation,
    \item $\sim_a$ is an ``indistinguishability'' equivalence relation on the set $P_a=\{w\in W\mid aPw\}$ for each agent $a\in \mathcal{A}$,
    \item $\pi(p)\subseteq  P$ for each propositional variable $p$.
\end{enumerate}
\end{definition}

In addition to the notation $P_a$, introduced above, it is also convenient to use the notation
$P_w=\{a\in \mathcal{A}\;|\;aPw\}$.

In our ``museum'' running example, set $W$ consists of worlds $w_1$ and $w_2$. Set $\mathcal{A}$ is $\{a,b,c\}$. Presence relation $P$ consists of all pairs from the set $\mathcal{A}\times W$ except for $(b,w_2)$ because agent $b$ is not present in world $w_2$. In the same example, $w_1\sim_a w_2$. The relations $\sim_b$ and $\sim_c$ are not important for that example. If propositional variable $p$ represents the statement ``is a WeRide'', then $\pi(p)=\{(b,w_1),(c,w_2)\}$.

\section{Syntax and Semantics}

The language $\Phi$ of our logical system is defined by the grammar:
$$
\phi:=p\mid\neg\phi\mid\phi\to\phi\mid\K\phi\mid\R\phi\mid \D\phi,
$$
where $p$ is a propositional variable. We read $\K\phi$ as ``knows $\phi$ about herself'', $\R\phi$ as ``de re aware about someone with property $\phi$'', and $\D\phi$ as ``de dicto aware about someone with property $\phi$''. We assume that conjunction $\wedge$ and disjunction $\vee$ as well as constants truth $\top$ and false $\bot$ are defined in the standard way.

\begin{definition}\label{sat} 
For any world $w\in W$, any agent $a\in P_w$ of an epistemic model $(W,\mathcal{A},P,\sim,\pi)$, and any formula $\phi\in \Phi$, the satisfaction relation $w,a\Vdash \phi$ is defined recursively as follows:
\begin{enumerate}
    \item $w,a\Vdash p$ if $(a,w)\in \pi(p)$,
    \item $w,a\Vdash \neg \phi$ if $w,a\nVdash \phi$,
    \item $w,a\Vdash \phi\to\psi$ if  $w,a\nVdash \phi$  or $w,a\Vdash \psi$, 
    \item $w,a\Vdash \K\phi$ if $u,a\Vdash\phi$, for each world $u\in P_a$ such that $w\sim_a u$,
    \item $w,a\Vdash \R\phi$ if there is such an agent $b\in P_w$ that 
    \begin{enumerate}
        \item $w,b\Vdash \phi$ and
        \item for any world $u\in P_a$ if $w\sim_a u$, then $u\in P_b$,
    \end{enumerate}
    \item $w,a\Vdash \D\phi$ if for each world $u\in P_a$ such that $w\sim_a u$ there is an agent $b\in P_u$ such that $u,b\Vdash \phi$.
\end{enumerate}
\end{definition}
Note that item~4 above requires property $\phi$ to be true {\bf about} agent $a$ in all worlds indistinguishable {\bf by} agent $a$ from the current world. Thus, modality $\K\phi$ captures the knowledge of $\phi$ {by} agent $a$ {about} herself.

Item~5 above states that agent $b$ has property $\phi$ in the current world $w$ and agent $b$ is present in all worlds indistinguishable by agent $a$ from the current world. In other words, agent $b$ has property $\phi$ and agent $a$ is aware of agent $b$. In this case, we say that agent $a$ is {\em de re} aware of $\phi$. In our example in Figure~\ref{intro figure}, agent $c$ has the property ``is a police vehicle near the museum'' in the current world $w_1$ and agent $c$ is present in all worlds indistinguishable by Ann from the current world. 

Item~6 above states that in each world indistinguishable by $a$ from the current world, there is at least one agent with property $\phi$. Thus, agent $a$ is {\em de dicto} aware of $\phi$. In our running example, there is a WeRide vehicle near the museum in each of the worlds indistinguishable by Ann from the current world $w_1$.

In the philosophy of language, the type of semantics that we gave in Definition~\ref{sat} is sometimes called a 2D-semantics~\cite{s21sep}.
Modality $\K$ for such semantics has been studied before~\cite{slg11lia,slg13tark,en21aaai,ent23jlc,nt23jsl}. Modalities $\R$ and $\D$ are original to this paper.

\section{Axiomatisation}

In addition to the tautologies in language $\Phi$, our logical system has the following axioms:
\begin{enumerate} 
    \item Truth: $\K\phi\to\phi$,
    \item Negative Introspection: $\neg\K\phi\to\K\neg\K\phi$,
    \item Distributivity: $\K(\phi\to\psi)\to(\K\phi\to\K\psi)$,
    \item Self-Awareness: $\phi\to \R\phi$ and $\K\phi\to \D\phi$, 
    \item Introspection of Awareness: $\D\phi\to \K\D\phi$,
    \item Unawareness of Falsehood: $\neg \R\bot$ and $\neg\D\bot$,
    \item Disjunctivity: $\R(\phi\vee\psi)\to \R\phi\vee\R\psi$,
    \item General Awareness: $\D(\R\phi\vee \D\phi)\to \D\phi$. 
\end{enumerate}
The first three axioms are the standard axioms of the epistemic logic. It is easy to see that they hold for ``knows about herself'' modality $\K$. The Positive Introspection principle also holds for $\K$. We derive it from our axioms in Lemma~\ref{positive introspection lemma} in the appendix.

Note that Definition~\ref{sat} requires agent $a$ to be present in the world $w$ each time when $w,a\Vdash\phi$. Also, by Definition~\ref{epistemic model definition}, relation $\sim_a$ is defined only on the worlds in which agent $a$ is present. As a result, each agent is present in all worlds that she cannot distinguish from the current world. Thus, each agent is aware of her own presence in the current world. Then, if agent $a$ has property $\phi$, then $a$ is {\em de re} aware of $\phi$. We capture this observation in the first Self-Awareness axiom.
If agent $a$ has property $\K\phi$, then $\phi$ is true about $a$ in all worlds indistinguishable by $a$ from the current world. Hence, if agent $a$ has property $\K\phi$, then $a$ is {\em de dicto} aware of $\phi$. We capture this in the second Self-Awareness axiom. 

By item~6 of Definition~\ref{sat}, an agent is de dicto aware of $\phi$ if each indistinguishable world contains an agent with property $\phi$. Thus, if formula $\D\phi$ is true, then this formula must also be true in all indistinguishable worlds. We state this in the Introspection of Awareness axiom. A similar axiom for de re modality $\R$, generally speaking, is not valid. 

Note that items~5 and 6 of Definition~\ref{sat} require that formula $\phi$ must be true about agent $b$ in the current world. Thus, an agent cannot be either de re or de dicto aware of a falsehood. We state this in the two Unawareness of Falsehood axioms.

Item~5(a) of Definition~\ref{sat} requires that for $w,a\Vdash\R(\phi\vee\psi)$ to be true, formula $\phi\vee\psi$ must be true in the current world about some agent $b$. Then either $\phi$ or $\psi$ must be true about $b$ in {\em the current world}. As a result, either statement $w,a\Vdash\R\phi$ or $w,a\Vdash\R\psi$ must be true. This justifies the Disjunctivity axiom.

Note that if, in each indistinguishable world, there is someone who is (either de re or de dicto) aware of a WeRide, then there must be a WeRide in each indistinguishable world. We capture this observation in the General Awareness axiom. We discuss the axiom's name in the section on $\lambda$-assured sets.

We write $\vdash \phi$ and say that formula $\phi\in\Phi$ is a {\em theorem} of our logical system if $\phi$ is derivable from the above axioms using the Modus Ponens, the Necessitation, and the two forms of the Monotonicity inference rules:
$$
\dfrac{\phi, \phi\to\psi}{\psi}
\hspace{10mm}
\dfrac{\phi}{\K\phi}
\hspace{10mm}
\dfrac{\phi\to \psi}{\D\phi\to \D\psi}
\hspace{10mm}
\dfrac{\phi\to \psi}{\R\phi\to \R\psi}.
$$

In addition to unary relation $\vdash\phi$, we also consider a binary relation $X\vdash \phi$ between a set of formulae $X\subseteq\Phi$ and a formula $\phi\in\Phi$. We say that $X\vdash \phi$ is true if formula $\phi$ is derivable from the {\em theorems} of our logical system and the set of additional axioms $X$ using {\em only} the Modus Ponens inference rule. It is easy to see that the statements $\varnothing\vdash\phi$ and $\vdash\phi$ are equivalent. We say that the set of formulae $X$ is consistent if $X\nvdash\bot$. The theorem below captures our informal discussion above. Formally, it follows from Definition~\ref{sat}.

\begin{lemma}[Lindenbaum]\label{Lindenbaum's lemma}
Any consistent set of formulae can be extended to a maximal consistent set of formulae.
\end{lemma}
\begin{proof}
The standard proof of Lindenbaum's lemma~\cite[Proposition 2.14]{m09} applies here.
\end{proof}

\section{Soundness}
\begin{theorem}[soundness]
If $\vdash\phi$, then $w,a\Vdash \phi$ for each world $w$ and each agent $a\in P_w$ of each epistemic model.    
\end{theorem}
The soundness of the Truth, the Negative Introspection, the Distributivity, the Self-Awareness, the Introspection of Awareness, the Unawareness of Falsehood, and the Disjunctivity axioms as well as of the inference rules are straightforward. Below, we prove the soundness of the General Awareness axiom as a separate lemma.  






\begin{lemma}
If $w,a\Vdash \D(\R\phi\vee\D\phi)$, then $w,a\Vdash \D\phi$.    
\end{lemma}
\begin{proof}
Suppose that $w,a\nVdash \D\phi$. Thus, by item~6 of Definition~\ref{sat}, there exists a world $u\in W$ such that 
\begin{equation}\label{28-dec-a}
w\sim_a u    
\end{equation}
and
\begin{equation}\label{28-dec-b}
\forall b\in P_u\;(u,b\nVdash \phi).    
\end{equation}
At the same time, by the assumption of the lemma and the same item~6 of Definition~\ref{sat}, there exists an agent $c\in P_u$ such that
$u,c\Vdash \R\phi\vee\D\phi$.
Thus, one of the following two cases takes place:

\vspace{1mm}\noindent
{\em Case I:} $u,c\Vdash \R\phi$. Hence, by item~5(a) of Definition~\ref{sat}, there exists an agent $d\in P_u$ such that $u,d\Vdash \phi$, which contradicts statement~\eqref{28-dec-b}.

\vspace{1mm}\noindent
{\em Case II:}  $u,c\Vdash \D\phi$. Hence, by item~6 of Definition~\ref{sat} and statement~\eqref{28-dec-a}, there exists an agent $d\in P_u$ such that $u,d\Vdash \phi$, which contradicts statement~\eqref{28-dec-b}.
\end{proof}

In the rest of the paper, we prove the completeness of our logical system.

\section{$\lambda$-assured sets}

In this section, we introduce a technical notion of a $\lambda$-assured set that will be used in the next section.
Throughout the rest of the paper, we use the notation $\A\phi$ to denote the formula $\R\phi\vee \D\phi$.
We read $\A\phi$ as ``is generally aware of $\phi$''. The General awareness axiom is essentially using this modality and is named after it. 

By $\A^n\phi$ we mean the formula $\underbrace{\A\dots\A}_\text{$n$ times}\phi$. In the special case $n=0$, the notation $\A^n\phi$ denotes formula $\phi$.

\begin{definition}\label{assured}
A set $X$ of formulae is $\lambda$-{\bf\em assured}
if $X\nvdash \A^n\neg\lambda$ for each $n\ge 0$.   
\end{definition}
To develop an intuition about $\lambda$-assured sets, let $X$ be a maximal consistent set of formulae and $\lambda$ be the property ``is not a spy''. Thus, $\neg\lambda$ is the property ``is a spy''. The formula $\A\neg\lambda$ means the agent is generally aware of a spy. The formula $\neg\A\neg\lambda$ means that the agent is not aware of a spy. The formula $\neg\A\A\neg\lambda$ means that the spy is embedded so well that the agent is not even aware of anyone who is aware of a spy. The formula $\neg\A\A\A\neg\lambda$ allows only the existence of ``super spies'' who are hidden so well that the agent is not aware of anyone aware of anyone aware of a spy. The notion of $\lambda$-assurance captures the fact that only the existence of absolutely undetectable ``ghost'' spies is consistent with set $X$. In other words, it says that all ``detectable'' agents in the setting captured by set $X$ must have property $\lambda$.

\section{Completeness}

\subsection{Frames}

Traditionally, proofs of the completeness in modal logic use a canonical model in which worlds are defined as maximal consistent sets of formulae. At the core of such proofs is a ``truth'' lemma stating that $\phi\in w$ if and only if $w\Vdash \phi$. This approach is not easy to apply to 2D-semantics as it requires a ``decoupling'' of a maximal consistent set into a world and an agent. In this paper, we use the ``matrix'' technique for such decoupling recently proposed by~\citet{nt23jsl}. The technique consists of building the canonical model as a matrix, whose rows correspond to worlds and whose columns correspond to agents. The elements of the matrix are maximal consistent sets representing all formulae that are satisfied in a given world-agent combination. \cite{nt23jsl} proves the completeness of a logical system for ``telling apart'' modality. This system contains modality $\K$, but it does not contain awareness modalities and it does not deal with de re/de dicto distinction. 

In this paper, we adapt the matrix technique~\citet{nt23jsl} in a novel way: frames include an explicit awareness relation, row labels $\lambda_w$, and the requirement that each $X_{wa}$ be $\lambda_w$-assured. To emphasise these additions, we refer to our matrices as ``frames''.

\begin{definition}\label{frame} 
\!\! A frame is a tuple
$(\alpha,\beta,\lambda,P,X,\sim,\rightsquigarrow)$, where
\begin{enumerate}
    \item $\alpha$, $\beta$ are ordinals and $\lambda_w\in\Phi$ is a formula for each $w<\alpha$,
    \item $P\subseteq \alpha\times\beta$ is a ``presence'' relation; we read $(w,a)\in P$ as ``agent $a$ is present at world $w$''; we slightly abuse the notations and for each $w<\alpha$ and each $a<\beta$ by $P_w$ and $P_a$ we denote the set $\{b<\beta\mid (w,b)\in P\}$ and the set
    $\{u<\alpha\mid (u,a)\in P\}$, respectively,
    \item $X$ is a function that maps each pair $(w,a)\in P$ into a $\lambda_w$-assured maximal consistent set of formulae denoted by $X_{wa}$,
    \item $\sim_a$ is an ``indistinguishability'' equivalence relation  on the set $P_a$ for each $a<\beta$ such that for any $w,u<\alpha$ and any formula $\phi\in\Phi$,
    \begin{enumerate}
    \item if $w\sim_a u$, then $\K\phi\in X_{wa}$ iff $\K\phi\in X_{ua}$,
    \end{enumerate}
    \item $\rightsquigarrow_w$ is a reflexive ``awareness'' relation on the set $P_w$ for each $w<\alpha$ such that for any $u<\alpha$, any $a,b<\beta$, and any formula $\phi\in\Phi$,
    \begin{enumerate}
        \item if $a\rightsquigarrow_w b$ and $w\sim_a u$, then $a\rightsquigarrow_u b$,
        \item if $a\rightsquigarrow_w b$ and $\R\phi\notin X_{wa}$, then $\phi\notin X_{wb}$.
    \end{enumerate} 
\end{enumerate}
\end{definition}

In linear algebra, matrices usually have a finite number of rows and a finite number of columns. In this paper, we allow infinite matrices with $\alpha$ rows and $\beta$ columns, where $\alpha$ and $\beta$ are two ordinals. Recall that elements of an ordinal $\alpha$ are ordinals smaller than $\alpha$. For example, $0=\varnothing$, $1=\{0\}$, $2=\{0,1\}$, \dots $\omega=\{0,1,2,\dots\}$, $\omega+1=\{0,1,2,\dots,\omega\}$. If a matrix has $\alpha$ rows, then we assume that the rows are indexed by the elements of ordinal $\alpha$. For example, a three-row matrix has row 0, row 1, and row 2.

Informally, a ``matrix'' is usually defined as a table.  Formally, a matrix is a function $X$ on the Cartesian product of the set of rows and the set of columns. Following the tradition, we use the notation $X_{wa}$ to denote the value of the matrix function $X$ on the pair $(w,a)$. Note that we define $X$ as a {\em total} function on the set $\alpha\times\beta$. Thus, $X_{w,a}$ is defined even if $(w,a)\notin P$. This is done only to avoid constant references to the domain of $X$ in the proofs. If $(w,a)\notin P$, then it is not significant for our proof which exactly $\lambda_w$-assured maximal consistent set is~$X_{wa}$.

There is a significant difference between the way awareness is treated in epistemic models (Definition~\ref{epistemic model definition}) and frames (Definition~\ref{frame}). Intuitively, an agent $a$ is ``aware'' of an agent $b$ in an epistemic model if agent $b$ is present in all worlds that agent $a$ cannot distinguish from the current world.  As we will see later, a frame represents a {\em partially constructed} model. Thus, some of the worlds (rows) might be missing and they will be added later. If we attempt to define awareness in frames the same way as it is done in epistemic models, an agent might become ``unaware'' of another agent after a new possible world is introduced. To make our construction work properly, we want to avoid this ``loss of awareness effect''. 
This problem did not exist in work~\cite{nt23jsl} that does not deal with awareness. 
{\bf To achieve this goal, in this paper we equip our frames with an awareness relation $a\rightsquigarrow_w b$.} Intuitively, it means that in world $w$ agent $a$ is aware of agent $b$. Item 5(a) of Definition~\ref{frame} states that if agent $a$ is aware of $b$ in the current world, then $a$ is aware of $b$ in each indistinguishable world. 

Note that there are two distinct references to awareness in our frame. One of them is semantical: through relation $\rightsquigarrow_w$ on the columns. The other is syntactical, through modalities $\R$ and $\D$ occurring in the formulae from a set $X_{wa}$. Item 5(b) of Definition~\ref{frame} connects these two references to awareness. It states: if $a\rightsquigarrow_w b$ and 
$\phi\in X_{wb}$,
then $\R\phi\in X_{wa}$. That is: if $a$ is semantically aware of $b$ and $b$ has property $\phi$, then $a$ is syntactically (de re) aware of someone with property $\phi$. In Definition~\ref{frame}, we state this item in contrapositive form for ease of use.

\subsection{Complete Frames}

As briefly mentioned in the previous subsection, frames represent partially built models. In order to be convertible into a canonical model a frame must be complete.

\begin{definition}\label{complete frame}
A frame  $(\alpha,\beta,\lambda,P,X,\sim,\rightsquigarrow)$ is {\bf\em complete} if for each $(u,b)\in P$ and each formula $\phi\in\Phi$,
\begin{enumerate}
    \item if $\K\phi\notin X_{ub}$, then there is  $v\in P_b$ such that $u\sim_b v$ and $\phi\notin X_{vb}$,
    \item if $\R\phi\in X_{ub}$, then there is $c\in P_u$ such that $b\rightsquigarrow_u c$ and $\phi\in X_{uc}$,
    \item if $b\not\rightsquigarrow_u c$, then there is $v\in P_b$ such that $u\sim_b v$ and  $c\notin P_v$,
    \item if $\D\phi\notin X_{ub}$, then there is $v\in P_b$ such that $u\sim_b v$, and $\lambda_v=\neg\phi$,
    \item if $\D\phi\in X_{ub}$, then there is $c\in P_u$ and $\phi\in X_{uc}$.
\end{enumerate}
\end{definition}

A frame $(\alpha,\beta,\lambda,P,X,\sim,\rightsquigarrow)$ is {\bf finite} if ordinals $\alpha$ and $\beta$ are finite.
In Lemma~\ref{complete frame exists lemma}, we prove that any finite frame can be extended to a complete frame. The formal definition of an extension is below. 

\begin{definition}\label{extension definition}
A frame $(\alpha',\beta',\lambda',P',X',\sim',\rightsquigarrow')$  is an {\bf\em extension} of a frame $(\alpha,\beta,\lambda,P,X,\sim,\rightsquigarrow)$ if
\begin{enumerate}
    \item $\alpha\le \alpha'$ and $\beta\le \beta'$,
    \item $\lambda'_w=\lambda_w$ for each $w<\alpha$,
    \item $P'\cap (\alpha\times\beta)=P$,
    \item $X'_{wa}=X_{wa}$ for $(w,a)\in P$,
    \item $w_1\sim_aw_2$ iff $w_1\sim'_a w_2$ for each $a<\beta$ and each $w_1,w_2\in P_a$,
    \item $a_1\rightsquigarrow_w a_2$ iff $a_1\rightsquigarrow'_w a_2$ for each $w<\alpha$ and each $a_1,a_2\in P_w$.
\end{enumerate}
\end{definition}

The proofs of the next four lemmas can be found in the appendix.

\begin{lemma}\label{type 5 lemma}
For any finite frame $(\alpha,\beta,\lambda,P,X,\sim,\rightsquigarrow)$, any $(u,b)\in P$, and any formula $\D\phi\in X_{ub}$, there is an extension $(\alpha,\beta+1,\lambda',P',X',\sim',\rightsquigarrow')$ such that $\phi\in X'_{u\beta}$.  
\end{lemma}

\begin{lemma}\label{type 1 lemma}
    For any finite frame $(\alpha,\beta,\lambda,P,X,\sim,\rightsquigarrow)$, any $(u,b)\in P$, and any formula $\K\phi\notin X_{ub}$,  there is an extension $(\alpha+1,\beta,\lambda',P',X',\sim',\rightsquigarrow')$ such that
    (i) $u\sim'_b \alpha$,
    (ii) $\phi\not\in X'_{\alpha b}$, and
    (iii) for each $c<\beta$, if $b\not\rightsquigarrow_u c$, then $c\notin P'_\alpha$. 
\end{lemma}

\begin{lemma}\label{type 2 lemma}
For any finite frame $(\alpha,\beta,\lambda,P,X,\sim,\rightsquigarrow)$, any $(u,b)\in P$, and any formula $\R\phi\in X_{ub}$, there is an extension $(\alpha,\beta+1,\lambda',P',X',\sim',\rightsquigarrow')$ such that
$b\rightsquigarrow'_u\beta$ and
$\phi\in X'_{u\beta}$.      
\end{lemma}

\begin{lemma}\label{type 4 lemma}
For any finite frame $(\alpha,\beta,\lambda,P,X,\sim,\rightsquigarrow)$, any $(u,b)\in P$, and any formula $\D\phi\notin X_{ub}$,
there is an extension $(\alpha+1,\beta,\lambda',P',X',\sim',\rightsquigarrow')$ such that $u\sim'_b \alpha$ and $\lambda'_\alpha$ is equal to $\neg\phi$. 
\end{lemma}

We write $F\sqsubseteq F'$ if frame $F'$ is an extension of the frame~$F$. 
For any (finite or infinite) chain of frames $F_1\sqsubseteq F_2\sqsubseteq F_3\sqsubseteq F_4\sqsubseteq \dots$, where $F_i=(\alpha_i,\beta_i,\lambda_i,P_i,X_i,\sim_i,\rightsquigarrow_i)$, the limit $\lim_{i} F_i$ is the tuple $(\bigcup_i\alpha_i,\bigcup_i\beta_i,\bigcup_i\lambda_i,\bigcup_i P_i,\bigcup_i X_i,\bigcup_i \!\!\!\sim_i, \bigcup_i \!\!\!\rightsquigarrow_i)$. As usual, to compute the union of functions we treat them as functional relations (sets of pairs). The next lemma follows from Definition~\ref{extension definition}.

\begin{lemma}\label{limit lemma}
The limit of a chain of extensions $F_1\sqsubseteq F_2\sqsubseteq F_3\sqsubseteq \dots$ is an extension of the frame $F_1$.    
\end{lemma}

\begin{definition}\label{type 1}
A {\bf\em Type 1} requirement is a tuple $(u,b,\phi)$, where $u,b<\omega$ and $\phi\in\Phi$. In a given frame  $(\alpha,\beta,\lambda,P,X,\sim,\rightsquigarrow)$ this requirement is
\begin{enumerate}
    \item {\bf\em active} if $u<\alpha$, $b<\beta$, and $(u,b)\in P$,
    \item {\bf\em fulfilled} if it is active and
    $u$, $b$, and $\phi$ satisfy item~1 of Definition~\ref{complete frame}. 
\end{enumerate}  
\end{definition}
The definition of
{\bf Type 2}, {\bf Type 4}, and {\bf Type 5} requirements are identical to the one above except that they refer to item~2, item~4, and item~5 of Definition~\ref{complete frame}. 

\begin{definition}\label{type 3}
A {\bf\em Type 3} requirement is a tuple $(u,b,c)$, where $u,b,c<\omega$. In a given frame  $(\alpha,\beta,\lambda,P,X,\sim,\rightsquigarrow)$ this requirement is
\begin{enumerate}
    \item {\bf\em active} if $u<\alpha$, $b,c<\beta$, $(u,b)\in P$, and $(u,c)\in P$,
    \item {\bf\em fulfilled} if it is active and
    $u$, $b$, and $c$ satisfy item~3 of Definition~\ref{complete frame}. 
\end{enumerate}  
\end{definition}

\begin{lemma}\label{Genie lemma}
Any finite frame that has an active unfulfilled requirement (of any type), can be extended to a finite frame where the same requirement is fulfilled.      
\end{lemma}
\begin{proof}
For requirements of Type~1, Type~2, Type~4, and Type~5, the statement of the lemma follows from Lemma~\ref{type 1 lemma}, Lemma~\ref{type 2 lemma}, Lemma~\ref{type 4 lemma}, and Lemma~\ref{type 5 lemma}, respectively.

In the case of Type~3 requirement, notice that $\K\bot\to\bot$ is an instance of the Truth axiom. Thus, $\K\bot\notin X_{ub}$ for each $(u,b)\in P$ because set $X_{ub}$ is consistent. Therefore, the statement of the lemma follows from Lemma~\ref{type 1 lemma}, where $\phi$ is $\bot$.
\end{proof}

The next lemma follows from Definition~\ref{extension definition} and the definition of a ``fulfilled'' requirement.
\begin{lemma}\label{extension fullfilled}
If a requirement (of any of the five types) is active and fulfilled in a frame, then it is also fulfilled in any extension of the frame.
\end{lemma}

The lemma below follows from Definition~\ref{complete frame} and the definition of a ``fulfilled'' requirement.
\begin{lemma}\label{job finished}
If all active requirements (of all five types) are fulfilled in a frame, then the frame is complete.   
\end{lemma}

\begin{lemma}\label{complete frame exists lemma}
Any finite frame can be extended to a complete frame.    
\end{lemma}

\begin{proof}
Let $F$ be an arbitrary finite frame.
Observe that there are countably many requirements of each of the five types.
Let $r_1,r_2,r_3,\dots$ be an enumeration of all requirements of all five types (combined). We define a (possibly infinite) chain of {\em finite} frames $F_1\sqsubseteq F_2\sqsubseteq \dots$ recursively:
\begin{enumerate}
    \item $F_1=F$,
    \item if frame $F_n$ does not contain any active unfulfilled requirements, then $F_n$ is the last element of the chain,
    \item otherwise, let $r_{min}$ be the first (in terms of the enumeration $r_1,r_2,r_3,\dots$) active unsatisfied requirement in frame $F_n$; by Lemma~\ref{Genie lemma}, frame $F_n$ can be extended to a finite frame $F_{n+1}$ that fullfils requirement $r_{min}$.  
\end{enumerate}
\begin{claim}
Frame $\lim_n F_n$ is complete.    
\end{claim}
\begin{proof-of-claim}
Consider any requirement $r$ (of any of the five types). By Lemma~\ref{job finished}, it suffices to show that if requirement $r$ is active in frame $\lim_n F_n$, then it is fulfilled.

Indeed, if $r$ is active in $\lim_n F_n$, then (by definition of being ``active'') $r$ must be active in frame $F_i$ for some $i\ge 0$. Observe that, due to the construction of the chain $F_1\sqsubseteq F_2\sqsubseteq\dots$ if requirement $r$ is active in frame $F_i$, then it is fulfilled in frame $F_j$ for some $j\ge i$. Therefore, requirement $r$ is fulfilled in frame $\lim_n F_n$ by Lemma~\ref{limit lemma} and Lemma~\ref{extension fullfilled}.  
\end{proof-of-claim}
Frame $\lim_n F_n$ is an extension of the frame $F_1=F$ by Lemma~\ref{limit lemma}.
\end{proof}

\subsection{Canonical Model}

For any given frame $(\alpha,\beta,\lambda,P,X,\sim,\rightsquigarrow)$ we consider an epistemic model $(\alpha,\beta,P,\sim,\pi)$, where 
\begin{equation}\label{canonical pi}
\pi(p)=\{(w,a)\mid p\in X_{wa}\}.    
\end{equation}
Note that, in particular, worlds of the model are the elements of $\alpha$ and agents are the elements of $\beta$. The next lemma connects the epistemic model and the frame on which it is based. This lemma plays the role of a ``truth'' lemma in the classical proofs of completeness.

\begin{lemma}\label{truth lemma}
If frame $(\alpha,\beta,\lambda,P,X,\sim,\rightsquigarrow)$ is complete, then 
$w,a\Vdash \phi$ iff $\phi\in X_{wa}$ for any world $w<\alpha$, any agent $a\in P_w$, and any formula $\phi\in\Phi$.   
\end{lemma}
\begin{proof}
We prove the lemma by induction on structural complexity of formula $\phi$. If $\phi$ is a propositional variable, then the statement of the lemma follows from statement~\eqref{canonical pi} and item~1 of Definition~\ref{sat}. If $\phi$ is a negation or an implication then the statement of the lemma follows from either item~2 or item~3 of Definition~\ref{sat}, the induction hypothesis, and the maximality and consistency of set $X_{wa}$ in the standard way.  

Suppose that formula $\phi$ has the form $\K\psi$.

\noindent $(\Rightarrow):$
Assume that $\K\psi\notin X_{wa}$. Thus, by item~1 of Definition~\ref{complete frame}, there is $u\in P_a$ such that $w\sim_a u$ and $\psi\notin X_{ua}$. Then, $u,a\nVdash \psi$ by the induction hypothesis.
Therefore, $w,a\nVdash\K\psi$ by item~4 of Definition~\ref{sat}.

\noindent $(\Leftarrow):$
Assume that $\K\psi\in X_{wa}$. Consider any $u<\alpha$ such that $w\sim_a u$. By item~4 of Definition~\ref{sat} it suffices to show that $u,a\Vdash\psi$. Indeed, the assumptions $\K\psi\in X_{wa}$ and $w\sim_a u$, by item 4(a) of Definition~\ref{frame}, imply that $\K\psi\in X_{ua}$. Then, $X_{ua}\vdash \psi$ by the Truth axiom and the Modus Ponens inference rule. Thus, $\psi\in X_{ua}$ because $X_{ua}$ is a maximal consistent set. Hence, $u,a\Vdash\psi$ by the induction hypothesis.

Suppose that formula $\phi$ has the form $\R\psi$.

\noindent $(\Rightarrow):$ Assume that $w,a\Vdash \R\psi$. Thus, by item~5 of Definition~\ref{sat}, there is an agent $b\in P_w$ such that two facts hold. First,
\begin{equation}\label{14-june-a}
    w,b\Vdash \psi.
\end{equation}
Second, for any world $u\in P_a$, if $w\sim_a u$, then $u\in P_b$. The latter, by the contraposition of item~3 of Definition~\ref{complete frame}, implies \begin{equation}\label{14-june-b}
a\rightsquigarrow_w b.    
\end{equation}
At the same time, by the induction hypothesis, statement~\eqref{14-june-a} implies $\psi\in X_{wb}$. Therefore, $\R\psi\in X_{wa}$ by statement~\eqref{14-june-b} and item~5(b) of Definition~\ref{frame} applied contrapositively.

\noindent $(\Leftarrow):$
Assume that $\R\psi\in X_{wa}$. Then, by item 2 of Definition~\ref{complete frame}, there is $b\in P_w$ such that $a\rightsquigarrow_w b$ and $\psi\in X_{wb}$. Thus, by the induction hypothesis,
\begin{equation}\label{14-june-c}
w,b\Vdash \psi.
\end{equation}
Furthermore, by item 5(a) of Definition~\ref{frame}, $a\rightsquigarrow_u b$ for every world $u\in P_a$ such that $w\sim_a u$.
Note that $\rightsquigarrow_u$ is a relation on set $P_u$. Thus, $b\in P_u$ for every world $u\in P_a$ such that $w\sim_a u$. In other words, 
$u\in P_b$ for every $u\in P_a$ such that $w\sim_a u$.
Therefore $w,a\Vdash \R\psi$ by equation~\eqref{14-june-c} and item 5 of Definition~\ref{sat}.

Suppose that formula $\phi$ has the form $\D\psi$.

\noindent $(\Rightarrow):$ Towards contradiction, assume $\D\psi\notin X_{wa}$. Thus, by item~4 of Definition~\ref{complete frame}, there is world $u\in P_a$ such that  
\begin{align}  
&w\sim_a u, \label{15-june-a}\\
&\lambda_u=\neg\psi.\label{15-june-b}
\end{align}

By the assumption $w,a\Vdash \D\psi$, statement~\eqref{15-june-a}, and item~6 of Definition~\ref{sat}, there exists an agent $b\in P_u$ such that $u,b\Vdash \psi$. Hence, $\psi\in X_{ub}$ by the induction hypothesis. Then,
$X_{ub}\vdash\neg\neg\psi$ by the laws of propositional reasoning. Thus,
$X_{ub}\vdash\neg\lambda_u$ by equation~\eqref{15-june-b}.
In other words, $X_{ub}\vdash \A^0\neg\lambda_u$. Therefore, set $X_{ub}$ is not $\lambda_u$-assured by Definition~\ref{assured}, which contradicts item~3 of Definition~\ref{frame}.

\noindent $(\Leftarrow):$ We need to show that $w,a\Vdash \D\psi$. Consider a world $u\in P_a$ such that 
\begin{equation}\label{15-june-c}
w\sim_a u.    
\end{equation}
By item~6 of Definition~\ref{sat}, it suffices to show that there exists an agent $b\in P_u$ such that $u,b\Vdash \psi$.

Assume that $\D\psi\in X_{wa}$. Then, by the Introspection of Awareness axiom and the Modus Ponens inference rule, $X_{wa}\vdash \K\D\psi$. Thus, $\K\D\psi \in X_{wa}$ because $X_{wa}$ is a maximal consistent set. Hence, $\K\D\psi\in X_{ua}$ by item 4 of Definition~\ref{frame} and statement~\eqref{15-june-c}. Then, by the Truth axiom and the Modus Ponens inference rule, $X_{ua}\vdash \D\psi$. Thus, $\D\psi \in X_{ua}$ since $X_{ua}$ is a maximal consistent set. Hence, by item 5 of Definition~\ref{complete frame}, there exists an agent $b\in P_u$ such that $\psi\in X_{ub}$. Therefore, by the induction hypothesis, $u,b \Vdash \psi$.
\end{proof}

\begin{theorem}[strong completeness]
If $X\nvdash\phi$, then there is a world $w$ and an agent $a$ of an epistemic model such that $w,a\Vdash\chi$ for each formula $\chi\in X$ and $w,a\nVdash\phi$.      
\end{theorem}
\begin{proof}
Set $\{\neg\phi\}\cup X$ is consistent by the assumption $X\nvdash\phi$. By Lemma~\ref{Lindenbaum's lemma}, it can be extended to a maximal consistent set $X_{00}$. Consider tuple $F=(1,1,\lambda,P,X,\sim,\rightsquigarrow)$, where  
\begin{enumerate}
    \item $\lambda_0=\top$,
    \item $P=\{(0,0)\}$,
    \item $X(0,0)=X_{00}$,
    \item $\sim_0=\{(0,0)\}$,
    \item $\rightsquigarrow=\{(0,0)\}$.
\end{enumerate}
This tuple is a frame by Definition~\ref{frame}. By Lemma~\ref{complete frame exists lemma}, frame $F$ can be extended to a complete frame $F'$. Consider the canonical model corresponding to frame $F'$. Note that $0,0\Vdash\chi$ for each $\chi\in X$ and $0,0\Vdash \neg\phi$ by Lemma~\ref{truth lemma}. Therefore, $0,0\nVdash \phi$ by item~2 of Definition~\ref{sat}. 
\end{proof}

\section{Conclusion}

We have proposed to interpret ``awareness'' as knowledge of existence and observed that such knowledge can have two distinct forms: de re and de dicto. Our main technical result is a sound and complete logical system that describes the interplay between two modalities representing these two forms of awareness, as well as the standard ``knowledge of the fact'' modality usually studied in epistemic logic. 

\bibliography{naumov}

\clearpage

\appendix

\addtolength{\oddsidemargin}{2.5cm}
\addtolength{\evensidemargin}{2.5cm}
	\addtolength{\textwidth}{-5cm}



\onecolumn



\clearpage

\begin{center}
    {\LARGE\sc Technical Appendix}

\vspace{5mm}

    {\bf This appendix is not a part of the AAAI-26 proceedings.}
\end{center}

\vspace{5mm}

\begin{lemma}\label{A lemma}
$\vdash \D\A^n\phi\to\D\phi$ for each $n\ge 0$.   
\end{lemma}
\begin{proof}
We prove the lemma by induction on $n$. If $n = 0$, then $\D\A^n\phi\to\D\phi$ is a propositional tautology. Suppose that $\vdash \D\A^{n}\phi\to\D\phi$. Note that the formula $\D(\R\A^n\phi\vee \D\A^n\phi)\to \D\A^n\phi$ is an instance of the General Awareness axiom. Then $\vdash\D(\R\A^n\phi\vee \D\A^n\phi)\to \D\phi$ by propositional reasoning. Hence, $\vdash\D\A^{n+1}\phi\to \D\phi$ by the definition of notation $\A$.
\end{proof}

\begin{lemma}[Deduction]\label{deduction lemma}
If $X,\phi\vdash\psi$, then $X\vdash\phi\to\psi$.
\end{lemma}
\begin{proof}
Suppose that sequence $\psi_1,\dots,\psi_n$ is a proof from set $X\cup\{\phi\}$ and the theorems of our logical system that uses the Modus Ponens inference rule only. In other words, for each $k\le n$, either
\begin{enumerate}
    \item $\vdash\psi_k$, or
    \item $\psi_k\in X$, or
    \item $\psi_k$ is equal to $\phi$, or
    \item there are $i,j<k$ such that formula $\psi_j$ is equal to $\psi_i\to\psi_k$.
\end{enumerate}
It suffices to show that $X\vdash\phi\to\psi_k$ for each $k\le n$. We prove this by induction on $k$ through considering the four cases above separately.

\vspace{1mm}
\noindent{\em Case I}: $\vdash\psi_k$. Note that $\psi_k\to(\phi\to\psi_k)$ is a propositional tautology, and thus, is an axiom of our logical system. Hence, $\vdash\phi\to\psi_k$ by the Modus Ponens inference rule. Therefore, $X\vdash\phi\to\psi_k$. 

\vspace{1mm}
\noindent{\em Case II}: $\psi_k\in X$. Then, $X\vdash\psi_k$, similarly to the previous case.

\vspace{1mm}
\noindent{\em Case III}: formula $\psi_k$ is equal to $\phi$. Thus, $\phi\to\psi_k$ is a propositional tautology. Then, $X\vdash\phi\to\psi_k$. 

\vspace{1mm}
\noindent{\em Case IV}:  formula $\psi_j$ is equal to $\psi_i\to\psi_k$ for some $i,j<k$. Thus, by the induction hypothesis, $X\vdash\phi\to\psi_i$ and $X\vdash\phi\to(\psi_i\to\psi_k)$. Note that formula 
$$
(\phi\to\psi_i)\to((\phi\to(\psi_i\to\psi_k))\to(\phi\to\psi_k))
$$
is a propositional tautology. Therefore, $X\vdash \phi\to\psi_k$ by applying the Modus Ponens inference rule twice.
\end{proof}

\begin{lemma}[Positive Introspection]\label{positive introspection lemma}
$\vdash \K\phi\to\K\K\phi$. 
\end{lemma}
\begin{proof}
Formula $\K\neg\K\phi\to\neg\K\phi$ is an instance of the Truth axiom. Thus, $\vdash \K\phi\to\neg\K\neg\K\phi$ by contraposition. Hence, taking into account the following instance of  the Negative Introspection axiom: $\neg\K\neg\K\phi\to\K\neg\K\neg\K\phi$,
we have 
\begin{equation}\label{pos intro eq 2}
\vdash \K\phi\to\K\neg\K\neg\K\phi.
\end{equation}

At the same time, $\neg\K\phi\to\K\neg\K\phi$ is an instance of the Negative Introspection axiom. Thus, $\vdash \neg\K\neg\K\phi\to \K\phi$ by the law of contrapositive in the propositional logic. Hence, by the Necessitation inference rule, 
$\vdash \K(\neg\K\neg\K\phi\to \K\phi)$. Thus, by  the Distributivity axiom and the Modus Ponens inference rule, 
$
  \vdash \K\neg\K\neg\K\phi\to \K\K\phi.
$
The latter, together with statement~(\ref{pos intro eq 2}), implies the statement of the lemma by propositional reasoning.
\end{proof}

\begin{lemma}\label{super distributivity}
If $\phi_1,\dots,\phi_n\vdash\psi$, then $\K\phi_1,\dots,\K\phi_n\vdash\K\psi$.
\end{lemma}
\begin{proof}
We prove the statement of the lemma by induction on $n$. If $n=0$, then the statement of the lemma follows from the Necessitation inference rule.

Suppose $n>0$. By Lemma~\ref{deduction lemma}, the assumption
$$\phi_1,\dots,\phi_{n-1}, \phi_{n}\vdash\psi.$$
implies
$$\phi_1,\dots,\phi_{n-1}\vdash \phi_{n}\to\psi.$$
Then, by the induction hypothesis, 
$$\K\phi_1,\dots,\K\phi_{n-1}\vdash \K(\phi_{n}\to\psi).$$
Hence, by the Distributivity axiom and the Modus Ponens inference rule,
$$\K\phi_1,\dots,\K\phi_n\vdash \K\phi_{n+1}\to\K\psi.$$
Therefore, 
$$\K\phi_1,\dots,\K\phi_n, \K\phi_{n+1}\vdash\K\psi.$$
by Modus Ponens inference rule.
\end{proof}

\begin{lemma}\label{unawareness of falsehood}
$\vdash \neg\A^n\neg\top$ for each $n\ge 0$.    
\end{lemma}
\begin{proof}
We prove the lemma by induction on $n$. If $n = 0$, then $\neg\A^n\neg\top$ is a propositional tautology. Assume that $\vdash \neg\A^{n}\neg\top$. Then, $\vdash \A^{n}\neg\top\to \bot$ by propositional reasoning. Thus, by the Monotonicity inference rule, $\vdash \R\A^{n}\neg\top \to \R\bot$. Hence, $\vdash \R\A^{n}\neg\top \to \bot$ by the Unawareness of Falsehood axiom and propositional reasoning. Thus, by propositional reasoning,
\begin{equation}\label{2-august-a}
\vdash \neg\R\A^{n}\neg\top.
\end{equation}
At the same time, by the Monotonicity inference rule, $\vdash \D\A^{n}\neg\top \to \D\bot$. Then, by the Unawareness of Falsehood axiom and propositional reasoning $\vdash \D\A^{n}\neg\top \to \bot$. Hence,
$
\vdash \neg\D\A^{n}\neg\top.
$
Thus, by equations~\eqref{2-august-a} and propositional reasoning,  $\vdash \neg(\R\A^{n}\neg\top\vee \D\A^{n}\neg\top)$. Hence, $\vdash \neg\A\A^{n}\neg\top$ and consequently $\vdash \neg\A^{n+1}\neg\top$ by the definition of notation $\A$.
\end{proof}

\begin{corollary}\label{cons is top-assured}
Any consistent set of formulae is $\top$-assured.    
\end{corollary}

\begin{lemma}\label{Monotonicity-of-A}
$\vdash \A^m\phi\to\A^n\phi$, for any $n\ge m\ge 0$.
\end{lemma}
\begin{proof}
It suffices to prove that $\vdash\psi\to\A^k\psi$ for each $k\ge 0$ and each formula $\psi$. We show this by induction on $k$. If $k=0$, then the statement is true because $\psi\to\psi$ is a propositional tautology. 
Suppose that $\vdash\psi\to\A^k\psi$. Note that the formula $\A^k\psi\to\R\A^k\psi$ is an instance of the Self-Awareness axiom. Then, $\vdash\psi\to(\R\A^k\psi\vee\D\A^k\psi)$ by propositional reasoning. Therefore, $\vdash\psi\to\A^{k+1}\psi$ by the definition of  $\A$.
\end{proof}

\begin{lemma}[bridge builder]\label{bridge builder}
For any formula $\lambda\in\Phi$ and any $\lambda$-{assured} maximal consistent set of formulae $X$ and any formula $\R\phi\in X$, there is a $\lambda$-assured maximal consistent set $Y$ such that
$
\{\phi\}\cup  \{\neg \psi\mid \R\psi\notin X\}\subseteq Y
$.
\end{lemma}
\begin{proof}

\begin{claim}
Set     
$
\{\phi\}\cup  \{\neg \psi\mid \R\psi\notin X\}
$
is consistent.
\end{claim}
\begin{proof-of-claim}
Suppose the opposite. Thus, there are formulae
\begin{equation}\label{2-aug-a}
\R\psi_1,\dots,\R\psi_n\notin X    
\end{equation}
such that
$\neg\psi_1,\dots,\neg\psi_n\vdash \neg\phi$.
Hence, 
$\phi\vdash \psi_1\vee\dots\vee\psi_n$
by the laws of propositional reasoning. Thus,
$\vdash\phi\to \psi_1\vee\dots\vee\psi_n$
by Lemma~\ref{deduction lemma}.
Then,
$\vdash\R\phi\to \R(\psi_1\vee\dots\vee\psi_n)$
by the Monotonicity inference rule.
Thus, 
$\vdash\R\phi\to \R\psi_1\vee\dots\vee\R\psi_n$
by the Disjunctivity axiom and propositional reasoning.
Hence,
$X\vdash\R\psi_1\vee\dots\vee\R\psi_n$
by the assumption $\R\phi\in X$ of the lemma. Thus, 
$\R\psi_i\in X$ for some $i\le n$
because $X$ is a maximal consistent set, which contradicts statement~\eqref{2-aug-a}.
\end{proof-of-claim}
By Lemma~\ref{Lindenbaum's lemma}, the set 
$
\{\phi\}\cup  \{\neg \psi\mid \R\psi\notin X\}
$
can be extended to a maximal consistent set~$Y$. By Definition~\ref{assured}, to finish the proof of the lemma, it suffices to establish the following claim.
\begin{claim}
 $\neg\A^{n}\neg \lambda\in Y$ for each $n\ge 0$.   
\end{claim}
\begin{proof-of-claim}
Suppose that
$\neg\A^{n}\neg \lambda\notin Y$
for some $n\ge 0$.
Then, by the choice of set $Y$,
$$\neg\A^{n}\neg \lambda\notin\{\neg \psi\mid \R\psi\notin X\}.$$
Thus, $\R\A^{n}\neg \lambda \in X$ by considering $\psi=\A^{n}\neg \lambda$ in the above formula.
Hence, by the laws of propositional reasoning,
$X\vdash (\R\A^{n}\neg \lambda)\vee (\D\A^{n}\neg \lambda)$.
Thus, 
$X\vdash \A\A^{n}\neg \lambda$ by the definition of notation $\A$.
Then, 
$X\vdash \A^{n+1}\neg \lambda$.
Therefore, set $X$ is not $\lambda$-assured by Definition~\ref{assured}, which contradicts an assumption of the lemma.
\end{proof-of-claim}
This concludes the proof of the lemma.
\end{proof}

\noindent{\bf Lemma~\ref{type 5 lemma}}
{\em
For any finite frame $(\alpha,\beta,\lambda,P,X,\sim,\rightsquigarrow)$, any $(u,b)\in P$, and any formula $\D\phi\in X_{ub}$, there is an extension $(\alpha,\beta+1,\lambda',P',X',\sim',\rightsquigarrow')$ such that $\phi\in X'_{u\beta}$.  }  
\begin{proof}
Let $\lambda'_w=\lambda_w$ for each $w<\alpha$. Furthermore, let $P'$ be the relation $P\cup\{(u,\beta)\}$.
Consider the set of formulae
\begin{equation}\label{25-may-a}
Y^-=\{\phi\}\cup \{\neg\A^n\neg\lambda_u\mid n\ge 0\}.
\end{equation}
\begin{claim}
Set $Y^-$ is consistent.    
\end{claim}
\begin{proof-of-claim}
Suppose the opposite. Then, there are numbers $n_1,\dots,n_k$ such that
$$
\phi \vdash \A^{n_1}\neg\lambda_u\vee
\A^{n_2}\neg\lambda_u\vee\dots\vee
\A^{n_k}\neg\lambda_u.
$$
Consider any $n\ge 0$ such that $n\ge n_i$ for each $i\le k$. Then, 
$
\vdash \phi \to \A^n\neg\lambda_u
$
by Lemma~\ref{Monotonicity-of-A} and propositional reasoning. Thus, 
$
\vdash \D\phi \to \D\A^n\neg\lambda_u
$
by the Monotonicity rule.
Hence,
$
\vdash \D\phi \to \A\A^n\neg\lambda_u
$
by the definition of notation $\A$ and propositional reasoning. In other words,
$
\vdash \D\phi \to \A^{n+1}\neg\lambda_u
$.
Thus,
$
X_{ub}\vdash \A^{n+1}\neg\lambda_u
$
by the assumption $\D\phi\in X_{ub}$ of the lemma and the Modus Ponens inference rule.
Hence, set $X_{ub}$ is not $\lambda_u$-assured, which contradicts item~3 of Definition~\ref{frame}.
\end{proof-of-claim}
By Lemma~\ref{Lindenbaum's lemma}, set $Y^-$ can be extended to a maximal consistent set $Y$. Consider a partial function $X'$ defined by the following matrix 
$$
X'=
\begin{bmatrix}
X_{00} & \dots & X_{0,b-1} & X_{0b} & X_{0,b+1} &\dots &X_{0\beta-1} & X'_{0 \beta}\\
\vdots &&\vdots&\vdots&\vdots&&\vdots&\vdots\\
X_{u0} & \dots & X_{u,b-1} & X_{ub} & X_{u,b+1} &\dots &X_{u,\beta-1} & X'_{u \beta}\\
\vdots &&\vdots&\vdots&\vdots&&\vdots&\vdots\\
X_{\alpha-1,0} & \dots & X_{\alpha-1,b-1} & X_{\alpha-1,b} & X_{\alpha-1,b+1} &\dots &X_{\alpha-1,\beta-1}& X'_{\alpha-1, \beta}
\end{bmatrix}
$$
where
\begin{equation}\label{25-may-b}
X'_{w\beta}=
\begin{cases}
Y, & \text{if $w=u$},\\
\text{undefined}, &\text{otherwise}.
\end{cases}
\end{equation}

For $a\neq\beta$, let relation $\sim'_a$ on the set $P'_a$ be the relation $\sim_a$. Additionally, let $\sim_\beta$ on the set $\alpha$ be the relation $\{(u,u)\}$.

Finally, let relation $\rightsquigarrow'_u$ on set $P'_u$ be the relation $\rightsquigarrow_u\cup \{(\beta,\beta)\}$. For $w\neq u$, let $\rightsquigarrow'_w$ be the relation $\rightsquigarrow_w$. 

\begin{claim}
Tuple $(\alpha,\beta+1,\lambda',P',X',\sim',\rightsquigarrow')$ is a frame.    
\end{claim}
\begin{proof-of-claim} We prove conditions 3, 4(a), 5(a), and 5(b) of Definition~\ref{frame} separetely.

\noindent
{\em Condition 3}: Recall that $\lambda'_u=\lambda_u$. Thus, it suffices to show that set $X'_{u\beta}$ is $\lambda_u$-assured. By Definition~\ref{assured} and equation~\eqref{25-may-b}, it suffices to show that $Y\nvdash \A^n\neg\lambda_u$ for each $n\ge 0$. This follows from equation~\eqref{25-may-a} and the consistency of set $Y$.

\noindent
{\em Condition 4(a)}: By definition of $\sim'$, it suffices to show that $\K\psi\in  X'_{u\beta}$ iff $\K\psi\in  X'_{u\beta}$ for each formula $\psi\in\Phi$. The last statement is trivially true.

\noindent
{\em Condition 5(a)}: By definition of $\sim'$ and $\rightsquigarrow'$ , it suffices to show that $\beta\rightsquigarrow'_u \beta$. The latter is true again by definition of the relation $\rightsquigarrow'$.

\noindent
{\em Condition 5(b)}: It suffices to show that if $\R\psi\notin X'_{u\beta}$, then $\psi\notin X'_{u\beta}$ for each formula $\psi\in\Phi$. Recall that $X'_{u\beta}$ is a maximal consistent set by equation~\eqref{25-may-a} and the definition of $Y$. Hence, $X'_{u\beta}\vdash \neg \R\psi$. Then, by the contraposition of the Self-Awareness axiom, $X'_{u\beta}\vdash \neg\psi$. Therefore, $\neg\psi\not\in X'_{u\beta}$ because set $X'_{u\beta}$ is consistent.
\end{proof-of-claim}

To finish the proof of the lemma, note that the frame $(\alpha,\beta+1,\lambda',P',X',\sim',\rightsquigarrow')$ is an extension of the frame $(\alpha,\beta,\lambda,P,X,\sim,\rightsquigarrow)$ by Definition~\ref{extension definition}. Furthermore, $\phi\in Y^-\subseteq Y= X'_{u\beta}$ by equations~\eqref{25-may-a} and \eqref{25-may-b}.
\end{proof}

\noindent{\bf Lemma~\ref{type 1 lemma}}
{\em 
    For any finite frame $(\alpha,\beta,\lambda,P,X,\sim,\rightsquigarrow)$, any $(u,b)\in P$, and any formula $\K\phi\notin X_{ub}$,  there is an extension $(\alpha+1,\beta,\lambda',P',X',\sim',\rightsquigarrow')$ such that
    (i) $u\sim'_b \alpha$,
    (ii) $\phi\not\in X'_{\alpha b}$, and
    (iii) for each $c<\beta$, if $b\not\rightsquigarrow_u c$, then $c\notin P'_\alpha$.
}
\begin{proof} 
Let, for each $w<\alpha+1$,
\begin{equation}\label{1-may-a}
\lambda'_{w}=
\begin{cases}
\top & \text{if $w=\alpha$},\\
\lambda_{w} & \text{otherwise},\\
\end{cases}
\end{equation}
and
\begin{equation}\label{P' definition}
P'=P\cup \{(\alpha,c)\mid b\rightsquigarrow_u c\}.    
\end{equation}

Consider the set of formulae   
\begin{equation}\label{Y- def}
Y^-=\{\neg\phi\}\cup\{\psi\mid \K\psi\in X_{ub}\}.    
\end{equation}
\begin{claim}
Set $Y^-$ is consistent.     
\end{claim}
\begin{proof-of-claim}
Suppose the opposite. Then, there are formulae 
\begin{equation}\label{2-apr-24}
\K\psi_1,\dots,\K\psi_n\in X_{ub}    
\end{equation}
such that
$$\psi_1,\dots,\psi_n\vdash \phi.$$ 
Then, by Lemma~\ref{super distributivity},
$$\K\psi_1,\dots,\K\psi_n\vdash \K\phi.$$
Hence, $X_{ub}\vdash \K\phi$ by statement~\eqref{2-apr-24}.
Therefore, $\K\phi\in X_{ub}$ because $X_{ub}$ is a maximal consistent set of formulae, which contradicts the assumption $\K\phi\notin X_{ub}$ of the lemma.
\end{proof-of-claim}    
By Lemma~\ref{Lindenbaum's lemma}, set $Y^-$ can be extended to a maximal consistent set $Y$.

\begin{claim}\label{29-apr-claim}
$\K\psi\in X_{ub}$ iff $\K\psi\in Y$ for any formula $\psi\in\Phi$.   
\end{claim}
\begin{proof-of-claim}
$(\Rightarrow):$    
By Lemma~\ref{positive introspection lemma} and the Modus Ponens inference rule, the assumption
$\K\psi\in X_{ub}$
implies that
$X_{ub}\vdash\K\K\psi$.
Then,
$\K\K\psi\in X_{ub}$
because $X_{ub}$ is a maximal consistent set. Therefore, $\K\psi\in Y^-\subseteq Y$ by equation~\eqref{Y- def}.

\vspace{1mm}\noindent
$(\Leftarrow):$    
Suppose that
$\K\psi\notin X_{ub}$. 
Then,
$\neg\K\psi\in X_{ub}$
because $X_{ub}$ is a maximal consistent set. Thus, $X_{ub}\vdash \K\neg\K\psi$ by the Negative Introspection axiom and the Modus Ponens inference rule. Hence,
$\K\neg\K\psi\in X_{ub}$ again because $X_{ub}$ is a maximal consistent set.
Then,
$\neg\K\psi\in Y^-\subseteq Y$
by equation~\eqref{Y- def}.
Therefore, 
$\K\psi\notin Y$
because set $Y$ is a consistent set.
\end{proof-of-claim}

Note that the formula $\top\to\R\top$ is an instance of the Self-Awarness axiom. Thus, $\vdash\R\top$ by propositional reasoning. Hence, $\R\top\in Y$ because $Y$ is a maximal consistent set of formulae. At the same time, set $Y$ is $\top$-assured by Corollary~\ref{cons is top-assured}. Hence, by Lemma~\ref{bridge builder}, there is a maximal consistent set of formulae $Z$ such that
\begin{equation}\label{29-apr-b}
\{\neg\psi\mid \R\psi\notin Y\}\subseteq Z.
\end{equation}
Consider partial function $X'$ defined by the following matrix:
$$
X'=
\begin{bmatrix}
X_{00} & \dots & X_{0,b-1} & X_{0 b} & X_{0,b+1} &\dots &X_{0,\beta-1}\\
\vdots &&\vdots&\vdots&\vdots&&\vdots\\
X_{u0} & \dots & X_{u,b-1} & X_{u b} & X_{u,b+1} &\dots &X_{u,\beta-1}\\
\vdots &&\vdots&\vdots&\vdots&&\vdots\\
X_{\alpha-1,0} & \dots & X_{\alpha-1,b-1} & X_{\alpha-1,b} & X_{\alpha-1,b+1} &\dots &X_{\alpha-1,\beta-1}\\
X'_{\alpha 0} &\dots &X'_{\alpha,b-1}&X'_{\alpha b}&X'_{\alpha,b+1} &\dots &X'_{\alpha,\beta-1}
\end{bmatrix},
$$
where 
\begin{equation}\label{29-apr-a}
X'_{\alpha a}=
\begin{cases}
Y & \text{if $a=b$},\\
Z & \text{if $a\neq b$ and $b\rightsquigarrow_u a$},\\
\text{undefined} &\text{otherwise}.
\end{cases}    
\end{equation}

For $a\neq b$, let relation $\sim'_a$ on the set $P'_a$ be the reflexive closure of the relation $\sim_a$. Additionally, let relation $\sim'_b$ on the set $P'_b$ be the reflexive, symmetric, and transitive closure of the relation $\sim_b\cup\{(u,\alpha)\}$. 

For $w\neq \alpha$, let relation $\rightsquigarrow'_w$ on set $P'_w$ be the relation $\rightsquigarrow_w$. Furthermore, let relation $\rightsquigarrow'_\alpha$ on the set $\beta$ be the reflexive closure of the relation $\{(b,c)\mid\;  b\rightsquigarrow_u  c,\, c<\beta\}$.

This concludes the definition of the tuple $(\alpha+1,\beta,\lambda',P',X',\sim',\rightsquigarrow')$. 

\begin{claim}
Tuple $(\alpha+1,\beta,\lambda',P',X',\sim',\rightsquigarrow')$ is a frame.    
\end{claim}
\begin{proof-of-claim}
We prove conditions 3, 4(a), 5(a), and 5(b) of Definition~\ref{frame} separetely. 

\noindent
{\em Condition 3}: By Definition~\ref{assured}, it suffices to show that $X'_{\alpha a}\nvdash \A^n\neg\lambda'_\alpha$ for each $n\ge 0$ and each $a<\beta$. Indeed, recall that $\lambda'_\alpha=\top$ by equation~\eqref{1-may-a}. Then, $X'_{\alpha a}\vdash \neg\A^n\neg\lambda'_\alpha$ by Lemma~\ref{unawareness of falsehood}. 
Therefore, $X'_{\alpha a}\nvdash \A^n\neg\lambda'_\alpha$ because set $X'_{\alpha a}$ is consistent.

\noindent
{\em Condition 4(a)}:
It suffices to show that $\K\psi\in  X'_{u b}$ iff $\K\psi\in  X'_{\alpha b}$ for each formula $\psi\in\Phi$. The last statement follows from Claim~\ref{29-apr-claim} and equation~\eqref{29-apr-a}.

\noindent\vspace{1mm}
{\em Condition 5(a)}: It suffices to show that if $b\rightsquigarrow'_w c$, then $b\rightsquigarrow'_\alpha c$ for each $c<\beta$. The last statement follows from the definition of the relation $\rightsquigarrow'_\alpha$.

\noindent\vspace{1mm}
{\em Condition 5(b)}: It suffices to show that if $b\rightsquigarrow'_\alpha c$ and $\R\psi\notin X'_{\alpha b}$, then $\psi\notin X'_{\alpha c}$ for each $c<\beta$. If $b\neq c$, then, by equation~\eqref{29-apr-a}, it suffices to show that if $c\in P_\alpha$ and $\R\psi\notin Y$, then $\psi\notin Z$. The last statement follows from equation~\eqref{29-apr-b} and the consistency of set $Z$. If $b=c$, then by equation~\eqref{29-apr-a}, it suffices to show that if $\R\psi\notin Y$, then $\psi\notin Y$. The last statement follows from the Self-Awareness axiom, applied contrapositively, because $Y$ is a maximal consistent set.  
\end{proof-of-claim}

By Definition~\ref{extension definition}, the frame $(\alpha+1,\beta,\lambda',P',X',\sim',\rightsquigarrow')$ is an extension of the frame $(\alpha,\beta,\lambda,P,X,\sim,\rightsquigarrow)$.
Let us now show that the three items of the lemma are satisfied. Note that $u\sim'_b\alpha$ by the choice of relation $\sim'$. Next, observe that $\neg\phi\in Y^-\subseteq Y=X'_{\alpha\beta}$ by equations~\eqref{Y- def} and \eqref{29-apr-a}. Finally, for each $c<\beta$, if $b\not\rightsquigarrow_u c$, then $c\notin P'_\alpha$ by equation~\eqref{P' definition}.
\end{proof}

\noindent{\bf Lemma~\ref{type 2 lemma}}
{\em 
For any finite frame $(\alpha,\beta,\lambda,P,X,\sim,\rightsquigarrow)$, any $(u,b)\in P$, and any formula $\R\phi\in X_{ub}$, there is an extension $(\alpha,\beta+1,\lambda',P',X',\sim',\rightsquigarrow')$ such that
$b\rightsquigarrow'_u\beta$ and
$\phi\in X'_{u\beta}$.      
}
\begin{proof}
Let $\lambda'_w = \lambda_w$ for each $w<\alpha$. Let $P' = P\cup \{(w,\beta)\mid u\sim_b w\}$. 

By  item 3 of Definition~\ref{frame}, set $X_{wb}$ is $\lambda_w$-assured, for each $w\in P_b$. Note that $\R\top\in X_{wb}$ for each $w\in P_b$ by Self-Awareness axiom because $X_{wb}$ is a maximal consistent set. Also $\R\phi\in X_{ub}$ by the assumption of the lemma.
Then, by Lemma~\ref{bridge builder}, for every $w\in P_b$, there exists a 
\begin{equation}\label{11-may-a}
\text{$\lambda_w$-assured maximal consistent set $X'_{w\beta}$},  
\end{equation}
 such that
\begin{align}
&\{\top\}\cup \{\neg\psi\mid \R\psi\notin X_{wb}\}\subseteq X'_{w\beta},  \;\;\;\text{when  $w\in P'_\beta$ and  $w\neq u$},\label{11-may-b}\\ 
&\{\phi\}\cup\{\neg\psi\mid \R\psi\notin X_{ub}\}\subseteq X'_{u\beta}.  \label{11-may-c}
\end{align}
Furthermore, for for each $w<\alpha$, let
\begin{equation}
    X'_{wb} \text{ be undefined  when $w\notin P'_\beta$.}
\end{equation}

Consider partial function $X'$ defined by the following matrix:

$$
X'=
\begin{bmatrix}
X_{00} & \dots & X_{0,b-1} & X_{0b} & X_{0,b+1} &\dots &X_{0,\beta-1} & X'_{0\beta}\\
\vdots &&\vdots&\vdots&\vdots&&\vdots&\vdots\\
X_{u0} & \dots & X_{u,b-1} & X_{ub} & X_{u,b+1} &\dots &X_{u,\beta-1} & X'_{u \beta}\\
\vdots &&\vdots&\vdots&\vdots&&\vdots&\vdots\\
X_{\alpha-1,0} & \dots & X_{\alpha-1,b-1} & X_{\alpha-1,b} & X_{\alpha-1,b+1} &\dots &X_{\alpha-1,\beta-1}& X'_{\alpha-1, \beta}
\end{bmatrix}
$$

For $a\neq\beta$, let relation $\sim'_a$ on the set $P'_a$ be the relation $\sim_a$. Additionally, let $\sim_\beta$ on the set $\alpha$ be the relation $\{(w,w)\mid w\in P'_\beta\}$.

Finally, let relation $\rightsquigarrow'_w$ on set $P'_w$ be the reflexive closure of the relation $\rightsquigarrow_w\cup \{(b,\beta)\}$ for each $w<\alpha$ such that $u\sim'_b w$. Otherwise, $\rightsquigarrow'_w$ is the relation $\rightsquigarrow_w$. 

\begin{claim}
Tuple $(\alpha,\beta+1,\lambda',P',X',\sim',\rightsquigarrow')$ is a frame.    
\end{claim}
\begin{proof-of-claim}
We prove conditions 3, 4(a), 5(a), and 5(b) of Definition~\ref{frame} separetely.

\noindent
{\em Condition 3}: Recall that $\lambda'_w=\lambda_w$ for each $w<\alpha$. Thus, it suffices to show that set $X'_{w\beta}$ is $\lambda_w$-assured for each $w\in P'_{\beta}$. The last statement follows from statement~\eqref{11-may-a}. 

\noindent
{\em Condition 4(a)}: By definition of $\sim'$, it suffices to show that $\K\psi\in  X'_{w\beta}$ iff $\K\psi\in  X'_{w\beta}$ for each formula $\psi\in\Phi$. The last statement is trivially true.

\noindent
{\em Condition 5(a)}: It suffices to show that if $b\rightsquigarrow'_w\beta$ and $w\sim'_b v$, then $b\rightsquigarrow'_v\beta$. 
The assumption $b\rightsquigarrow'_w\beta$ implies 
$u\sim'_b w$ by the definition of the relation $\rightsquigarrow'_w$. Hence
$u\sim'_b v$ by the assumption $w\sim'_b v$.
Therefore, $b\rightsquigarrow'_v\beta$ again by the definition of the relation $\rightsquigarrow'_w$.

\noindent
{\em Condition 5(b)}: Since $X'_{wb} = X_{wb}$, it suffices to show that if $b\rightsquigarrow'_w \beta$ and $\R\psi\notin X_{wb}$, then $\psi\notin X'_{w\beta}$ for each $w\in P'_\beta$. The last statement follows from equations~\eqref{11-may-c} and~\eqref{11-may-b} and the consistency of set $X'_{u\beta}$.
\end{proof-of-claim}

To finish the proof of the lemma, note that the frame $(\alpha,\beta+1,\lambda',P',X',\sim',\rightsquigarrow')$ is an extension of the frame $(\alpha,\beta,\lambda,P,X,\sim,\rightsquigarrow)$ by Definition~\ref{extension definition}. Also, observe that
$b\rightsquigarrow'_u\beta$ by the definition of the relation $\rightsquigarrow'$. Finally, $\phi\in X'_{u\beta}$ by statement~\eqref{11-may-c}.
\end{proof}

\noindent{\bf Lemma~\ref{type 4 lemma}}
{\em 
For any finite frame $(\alpha,\beta,\lambda,P,X,\sim,\rightsquigarrow)$, any $(u,b)\in P$, and any formula $\D\phi\notin X_{ub}$,
there is an extension $(\alpha+1,\beta,\lambda',P',X',\sim',\rightsquigarrow')$ such that $u\sim'_b \alpha$ and $\lambda'_\alpha$ is equal to $\neg\phi$. 
}
\begin{proof}
Let, for each $w<\alpha+1$,
\begin{equation}\label{12-may-b}
\lambda'_{w}=
\begin{cases}
\neg\phi & \text{if $w=\alpha$},\\
\lambda_{w} & \text{otherwise},\\
\end{cases}
\end{equation}
and 
\begin{equation}\label{12-may-c}
P'=P\cup \{(\alpha,c)\mid b\rightsquigarrow_u c\}.    
\end{equation}

Consider the set of formulae
\begin{equation}\label{19-may-d}
Y^-=\{\neg\A^n\neg\lambda'_\alpha\mid n\ge 0\}\cup\{\psi\mid \K\psi\in X_{ub}\}.    
\end{equation}

\begin{claim}
Set $Y^-$ is consistent.     
\end{claim}
\begin{proof-of-claim}
Suppose the opposite. Then there are formulae
\begin{equation}\label{12-may-d}
    \K\psi_1,\ldots,\K\psi_m\in X_{ub}
\end{equation}
and numbers $n_1,\ldots, n_k$ such that
$$\psi_1, \ldots,\psi_m\vdash \A^{n_1}\neg\lambda'_\alpha\vee \A^{n_2}\neg\lambda'_\alpha\vee \dots \vee \A^{n_k}\neg\lambda'_\alpha.$$
Consider any $n\ge 0$ such that $n\ge n_i$ for each $i\le k$. Then, by Lemma~\ref{Monotonicity-of-A} and propositional reasoning,
$$\psi_1, \ldots,\psi_m\vdash \A^{n}\neg\lambda'_\alpha.$$
Thus, by equation~\eqref{12-may-b},
$$\psi_1,\dots,\psi_k\vdash \A^n\neg\neg\phi.$$
Then by Lemma~\ref{super distributivity}, 
$$\K\psi_1,\dots,\K\psi_k\vdash \K\A^n\neg\neg\phi.$$
By the Self-Awareness axiom and Modus Ponens rule, 
$$\K\psi_1,\dots,\K\psi_k\vdash \D\A^n\neg\neg\phi.$$
Hence, $X_{ub}\vdash  \D\A^{n}\neg\neg\phi$ by statement~\eqref{12-may-d}. Thus, 
\begin{equation}\label{18-may-a}
X_{ub}\vdash  \D\neg\neg\phi    
\end{equation}
by Lemma~\ref{A lemma} and the Modus Ponens inference rule. At the same time, by propositional tautology $\neg\neg\phi\to\phi$, by the Monotonicity inference rule, implies $\vdash\D\neg\neg\phi\to\D\phi$. Then, $X_{ub}\vdash  \D\phi$ by statement~\eqref{18-may-a} and the Modus Ponens inference rule. Therefore, because $X_{ub}$ is a maximal consistent set of formulae, $\D\phi\in X_{ub}$,  which contradicts the assumption $\D\phi\notin X_{ub}$ of the lemma.
\end{proof-of-claim}
By Lemma~\ref{Lindenbaum's lemma}, set $Y^-$ can be extended to a maximal consistent set $Y$. 
The proof of the next claim is the same as the proof of Claim~\ref{29-apr-claim}, but it uses equation~\eqref{19-may-d} instead of equation~\eqref{Y- def}.
\begin{claim}\label{19-may-cc}
$\K\psi\in X_{ub}$ iff $\K\psi\in Y$ for any formula $\psi\in\Phi$.   
\end{claim}

\begin{claim}\label{18-may-b}
Set $Y$ is $\neg\phi$-assured.    
\end{claim}
\begin{proof-of-claim}
Consider any $n\ge 0$. By Definition~\ref{assured}, it suffices to show that $Y\nvdash \A^n\neg\neg\phi$.

Note that $\neg\A^n\neg\lambda'_\alpha\in Y^-\subseteq Y$ by equation~\ref{Y- def}. Thus, $\neg\A^n\neg\neg\phi\in Y$ by equation~\eqref{12-may-b}. Therefore, $Y\nvdash \A^n\neg\neg\phi$ because set $Y$ is consistent.
\end{proof-of-claim}
Observe that $\top$ is a tautology. Then, $\vdash \R\top$ by the Self-Awareness axiom and the Modus Ponens inference rule. Hence, $\R\top\in Y$ because $Y$ is a maximal consistent set of formulae. Thus, by Lemma~\ref{bridge builder} and Claim~\ref{18-may-b} there is a
\begin{equation}\label{19-may-b}
\text{$\neg\phi$-assured maximal consistent set $Z$}   
\end{equation}
 such that
\begin{equation}\label{19-may-c}
\{\psi\mid \R\psi\notin Y\}\subseteq Z.    
\end{equation}
Consider partial function $X'$ defined by the following matrix:
$$
X'=
\begin{bmatrix}
X_{00} & \dots & X_{0,b-1} & X_{0b} & X_{0,b+1} &\dots &X_{0,\beta-1}\\
\vdots &&\vdots&\vdots&\vdots&&\vdots\\
X_{u0} & \dots & X_{u,b-1} & X_{ub} & X_{u,b+1} &\dots &X_{u,\beta-1}\\
\vdots &&\vdots&\vdots&\vdots&&\vdots\\
X_{\alpha-1,0} & \dots & X_{\alpha-1,b-1} & X_{\alpha-1,b} & X_{\alpha-1,b+1} &\dots &X_{\alpha-1,\beta-1}\\
X'_{\alpha 0} &\dots &X'_{\alpha,b-1}&X'_{\alpha b}&X'_{\alpha,b+1} &\dots &X'_{\alpha,\beta-1}
\end{bmatrix}
$$
where
\begin{equation}\label{19-may-a}
X'_{\alpha,a}=
\begin{cases}
Y & \text{if $a=b$},\\
Z & \text{if $a\neq b$ and $b\rightsquigarrow_u a$},\\
\text{undefined} &\text{otherwise}.
\end{cases}
\end{equation}

For $a\neq b$, let relation $\sim'_a$ on the set $P'_a$ be the reflexive closure of the relation $\sim_a$. Additionally, let relation $\sim'_b$ on the set $P'_b$ be the reflexive, symmetric, and transitive closure of the relation $\sim_b\cup\{(u,\alpha)\}$. 

For $w\neq \alpha$, let relation $\rightsquigarrow'_w$ on set $P'_w$ be the relation $\rightsquigarrow_w$. Furthermore, let relation $\rightsquigarrow'_\alpha$ on the set $\beta$ be the reflexive closure of the relation $\{(b,c)\mid\;  b\rightsquigarrow_u  c,\, c<\beta\}$.

\begin{claim}
Tuple $(\alpha+1,\beta,\lambda',P',X',\sim',\rightsquigarrow')$ is a frame.    
\end{claim}
\begin{proof-of-claim}
We prove conditions 3, 4(a), 5(a), and 5(b) of Definition~\ref{frame} separetely. 

\noindent
{\em Condition 3}: Recall that $\lambda'_w=\lambda_w$ for each $w<\alpha$. Thus, it suffices to show that set $X'_{\alpha a}$ is $\lambda'_\alpha$-assured for each $a\in P'_{\alpha}$. Note that sets $Y$ and $Z$ are $\neg\phi$-assured by Claim~\ref{18-may-b} and statement~\eqref{19-may-b} respectively. Then, by equation~\eqref{12-may-b}, sets $Y$ and $Z$ are $\lambda'_{\alpha}$-assured. Hence, by equation~\eqref{19-may-a}, set $X'_{\alpha a}$ is $\lambda'_\alpha$-assured.

\noindent
{\em Condition 4(a)}:
It suffices to show that $\K\psi\in  X'_{u b}$ iff $\K\psi\in  X'_{\alpha b}$ for each formula $\psi\in\Phi$. The last statement follows from Claim~\ref{19-may-cc} and equation~\eqref{19-may-a}.

\noindent\vspace{1mm}
{\em Condition 5(a)}: It suffices to show that if $b\rightsquigarrow'_w c$, then $b\rightsquigarrow'_\alpha c$ for each $c<\beta$. The last statement follows from the definition of the relation $\rightsquigarrow'_\alpha$.

\noindent\vspace{1mm}
{\em Condition 5(b)}: It suffices to show that if $b\rightsquigarrow'_\alpha c$ and $\R\psi\notin X'_{\alpha b}$, then $\psi\notin X'_{\alpha c}$ for each $c<\beta$. If $b\neq c$, then, by equation~\eqref{19-may-a}, it suffices to show that if $c\in P_\alpha$ and $\R\psi\notin Y$, then $\psi\notin Z$. The last statement follows from equation~\eqref{19-may-c} and the consistency of set $Z$. If $b=c$, then by equation~\eqref{19-may-a}, it suffices to show that if $\R\psi\notin Y$, then $\psi\notin Y$. The last statement follows from the Self-Awareness axiom, applied contrapositively, because $Y$ is a maximal consistent set.  
\end{proof-of-claim}

By Definition~\ref{extension definition}, the frame $(\alpha+1,\beta,\lambda',P',X',\sim',\rightsquigarrow')$ is an extension of the frame $(\alpha,\beta,\lambda,P,X,\sim,\rightsquigarrow)$. Note that $u\sim'_b \alpha$ by definition of the relation $\sim'$ and $\lambda'_\alpha$ is equal to $\neg\phi$ by equation~\eqref{12-may-b}.
\end{proof}

\end{document}